\documentclass[11pt, letter]{article}

\usepackage[utf8]{inputenc} 
\usepackage[T1]{fontenc}    
\usepackage{hyperref}       
\usepackage{url}            
\usepackage{booktabs}       
\usepackage{amsfonts}       
\usepackage{nicefrac}       
\usepackage{microtype}      

\usepackage{amsmath,amsfonts, amssymb, amsthm}
\usepackage{hyperref}

\newcommand{\eps}{\varepsilon}

\newcommand{\cA}{\ensuremath{\mathcal{A}}}

\newcommand{\cS}{\ensuremath{\mathcal{S}}}

\newcommand{\Aset}{\ensuremath{\mathcal{A}}}

\newcommand{\R}{\ensuremath{\mathbb{R}}}
\DeclareMathOperator*{\EE}{\mathbb{E}}
\DeclareMathOperator{\E}{\mathbb{E}}

\newcommand{\pihat}{\ensuremath{\widehat{\pi}}}

\newcommand{\wh}{\widehat}

\newcommand{\var}{\mathrm{Var}}

\newcommand{\Sset}{\ensuremath{\mathcal{S}}}

\newcommand{\Mhat}{\ensuremath{\widehat{M}}}

\newcommand{\norm}[1]{\lVert#1\rVert}

\newcommand{\Mopt}{\ensuremath{M_\text{opt}}}
\newcommand{\epsopt}{\ensuremath{\epsilon_{\text{opt}}}}
\newcommand{\poly}{\mathrm{poly}}
\newcommand{\cT}{\mathcal{T}}

\theoremstyle{plain}
\newtheorem{theorem}{\textbf{Theorem}}
\newtheorem{lemma}[theorem]{\textbf{Lemma}}

\theoremstyle{definition}
\newtheorem{definition}{Definition}

\newtheorem{proposition}{Proposition}

\renewcommand{\cite}{\citep}

\newcommand{\Scal}{\mathcal{S}}

\newcommand{\Acal}{\mathcal{A}}

\newcommand{\RR}{\mathbb{R}}

\usepackage{times}
\usepackage{enumitem}
\usepackage{fullpage}
\usepackage{algorithm, algorithmicx,algpseudocode}
\usepackage{ltablex}
\usepackage{natbib}
\usepackage[dvipsnames]{xcolor}

\renewcommand{\aa}[1]{\textsf{\color{blue} AA: #1}}

\usepackage{dsfont}
\newcommand{\wt}{\widetilde}

\title{Model-Based Reinforcement Learning with a Generative Model is Minimax Optimal}

\author{%
	Alekh Agarwal\\
	Microsoft\\
	\texttt{alekha@microsoft.com}
	\and
	Sham Kakade\\
	University of Washington\\
	\texttt{sham@cs.washington.edu}
	\and
	Lin F. Yang\\
	University of California, Los Angeles\\
	\texttt{linyang@ee.ucla.edu}\\
}

\begin{document}


\maketitle

\begin{abstract}

This work considers the sample and computational complexity of
obtaining an $\eps$-optimal policy in a discounted Markov Decision
Process (MDP), given only access to a generative model. In this model,
the learner accesses the underlying transition model via a sampling
oracle that provides a sample of the next state, when given any
state-action pair as input. This widely studied setting provides a natural abstraction which permits the investigation of sample-based planning over a long horizon, decoupled from the
complexity of exploration. In this work, we study the effectiveness of the most natural plug-in
approach to model-based planning: we build the maximum
likelihood estimate of the transition model in the MDP from
observations and then find an optimal policy in this empirical
MDP.  We ask arguably the most basic and unresolved
question in model based planning: is the na\"ive ``plug-in'' approach,
non-asymptotically, minimax optimal in the quality of the policy it
finds, given a fixed sample size? Here, the non-asymptotic regime
refers to when the sample size is sublinear in the model size.

With access to a generative model,
we resolve this question in the strongest possible sense: our main result shows that
\emph{any} high accuracy solution in the plug-in model constructed
with $N$ samples, provides an $\eps$-optimal policy in the true
underlying MDP (where $\eps$ is the minimax accuracy with $N$ samples
at every state, action pair).  In comparison, all prior (non-asymptotically)
minimax optimal results use model free approaches, such
as the Variance Reduced Q-value iteration algorithm
~\citep{NIPS2018_7765}, while the best known model-based results
(e.g. ~\cite{azar2013minimax}) require larger
sample sizes in their dependence on the planning
horizon or the state space.
Notably, we show that the model-based approach allows the use of \emph{any} efficient
planning algorithm in the empirical MDP, which simplifies algorithm
design as this approach does not tie the algorithm to the sampling procedure. The core of our analysis is a
novel ``absorbing MDP'' construction to address the statistical dependency issues that
arise in the analysis of model-based planning approaches, a
construction which may be helpful more generally.

\end{abstract}

\clearpage

\section{Introduction}
How best to plan across a long-horizon with access to an approximate
model of a Markov Decision Process? This is a fundamental question at
the heart of reinforcement learning, and understanding it is essential
to tackling even more complex challenges such as sample-efficient
exploration (see e.g. \cite{kakade2003sample, strehl2006pac,
  strehl2007probably, jaksch2010near, osband2014model,
  azar2017minimax, sidford2018variance, NIPS2018_7765}). When the
approximate model is arbitrary, these questions are studied, for
example, in the approximate dynamic programming
literature~\citep{bertsekas1976dynamic}. Before moving to
approximation questions, a more basic question is an information
theoretic one:
how many samples from the Markov Decision Process are required to yield a near optimal policy?
Our work studies this question in the generative model framework introduced in the work of
~\citet{kearns1999finite}.

In the generative model setting, the learning agent has
sampling access to a generative model of the Markov Decision Process
(henceforth MDP), and it can query the next state $s'$ sampled from
the transition process, given as input any state-action pair.  The information
theoretic question is to quantify
how many samples from the generative model are required in order to
obtain a near optimal policy; this question is analogous to the
classical question of \emph{sample complexity} in the supervised
learning setting.

Arguably, the simplest approach here is a \emph{model-based} one: the
approach is to first build the maximum
likelihood estimate of the transition model in the MDP from
observations and then find an optimal policy in this empirical
MDP.  This work seeks to address the following unresolved
question: is the na\"ive ``plug-in'' approach,
non-asymptotically, minimax optimal in the quality of the policy it
finds, given a fixed sample size? Throughout, we refer to the non-asymptotic regime
as one where the sample size is sublinear in the model size.
This work answers this question affirmatively
showing that a model based planning approach is non-asymptotically
minimax optimal.

We note that the first provably, non-asymptotically, minimax optimal
algorithm is the Variance Reduced Q-value iteration
algorithm~\citep{NIPS2018_7765}, a model free approach.  The significance of the
optimality of our model-based result is that it allows the use of
\emph{any} efficient planning algorithm in the empirical MDP, which
simplifies algorithm design due to that the
algorithm utilized need not be tied to the sampling procedure.  We now discuss our
contributions and the related work more broadly.

\subsection{Our Contributions}
There exists a large body of literature on MDPs and RL (see
e.g. \cite{kakade2003sample, strehl2009reinforcement,
  kalathil2014empirical, dann2015sample} and reference therein).  A
summary of our result relative to the prior works using a generative
model is presented in
Table~\ref{tab:literature_runtime_apx}.  Here, $\epsilon$ is a desired
accuracy parameter; $|\Scal|$ and $|\Acal|$ are the cardinalities of
the (finite) state and actions spaces; $\gamma$ is a discount
factor. We refer to $\eps$-optimal policy the one whose
discounted cumulative value in the MDP is $\eps$ close to the optimal value.

Before discussing the sample complexity of finding an $\eps$-optimal
policy, let us review the results on computing an $\eps$-optimal value
function. This refers to the problem of finding a function
$\widehat Q$ which approximates $Q^\star$ to an error of $\epsilon$ at
all states. The work of~\cite{azar2012sample} shows that for
$\epsilon \in (0,1)$ it suffices to use at most
$\wt{O}\big(\frac{|\Scal||\Acal|}{(1-\gamma)^3\epsilon^2}\big)$ calls
to the generative model in order to return an $\eps$-optimal value
function
\footnote{\label{footnote:self_bound}We conjecture that our techniques can be used to broaden the range of
$\eps$ to go beyond $\epsilon \in (0,1)$, as needed
in~\cite{azar2012sample}.
In particular, the proof of Lemma~\ref{lemma:final-bound} (used to prove Theorem~\ref{thm:sample_complexity}) uses a self-bounding approach
  which we conjecture can be used to broaden the range
  of $\eps$ to allow for  $\epsilon \in
  (0,\tfrac{1}{\sqrt{1-\gamma}}]$. We do not focus on this improvement
  in this work, as our main focus is on the value of the policy itself.}.
Furthermore,
the work of~\cite{azar2012sample} shows this sample complexity is
minimax optimal.

Obtaining an $\eps$-optimal policy (rather than just estimating the
value itself) is more subtle; na\"ively, a policy obtained in a greedy
manner from an
$\eps$-optimal value will incur a further degradation in its quality
by a factor of $1-\gamma$ \citep{singh1994upper}.
The work of ~\cite{azar2013minimax} shows that this additional error
amplification is avoidable provided that the number of samples is at
least $O(|\Scal|^2 |\Acal|)$ (see Table~\ref{tab:literature_runtime_apx}); note that such a sample size is actually
linear in the model size.

Our work avoids this error amplification and shows that for a desired accuracy threshold of
$\epsilon$, we can find an $\epsilon$-optimal policy for any
$\epsilon \in (0,\tfrac{1}{\sqrt{1-\gamma}}]$ using at most
$\wt{O}\big(\frac{|\Scal||\Acal|}{(1-\gamma)^3\epsilon^2}\big)$ samples. Our result holds for
\emph{any planning algorithm} that finds a near
optimal policy in the empirically constructed MDP. Due to existing
lower bounds~\citep{azar2012sample, NIPS2018_7765}, this bound is known to be
\emph{minimax optimal} for $\epsilon \in (0,1]$. Notably, this sample complexity is $o(S^2A)$ whenever $\epsilon^2 \geq 1/((1-\gamma)^3|\Sset|)$, meaning that we can use the model to find a near optimal policy even in sample regimes where no meaningfully accurate approximation to the actual transition probabilities can be constructed.

Prior to this work, the only other non-asymptotically minimax optimal
approach takes a different algorithmic path: \cite{NIPS2018_7765}
(also see \citet{sidford2018variance})
use a modification of the $Q$-value iteration method, with explicit
control of variance in value estimates, to obtain an optimal sample complexity for $\epsilon \in
(0,1]$. Our guarantees hold for a broader range of $\epsilon$ values (though
we conjecture that our techniques could also improve the
$\eps$ dependence in~\cite{NIPS2018_7765}. See Footnote~\ref{footnote:self_bound}.).

Importantly, our work highlights that the sub-optimality of the prior
model-based results was not due to any inherent limitation of the
approach, but instead due to a matter of analysis. As a by-product, we retain a
conceptually and algorithmically simpler solution strategy relative
to~\citet{NIPS2018_7765}.
On a more technical note, our analysis is based on a novel absorbing
MDP construction to deal with the dependence issues which arise in the
analysis of~\citet{azar2012sample, azar2013minimax}, and this argument might be more
broadly useful.

\begin{table*}[t!]
	\begin{center}
		\begin{tabular}{|p{4cm}|c|c|c|}
			\hline
			\textbf{Algorithm} & \textbf{Sample Complexity} &\textbf{$\epsilon$-Range} & \textbf{References} \\
			\hline
			& &
			\\[-1em]
			Phased Q-Learning& $C\frac{|\Scal| |\Acal|}{(1-\gamma)^7\epsilon^2}$ &$(0, (1-\gamma)^{-1}]$ & \cite{kearns1999finite}
			\\
			\hline
			& &
			\\[-1em]
			Empirical QVI& $\frac{|\Scal| |\Acal|}{(1-\gamma)^5\epsilon^2}$ &$(0,1]$& \cite{azar2013minimax}
			\\
			\hline
			&&
			\\[-1em]
			Empirical QVI & $\frac{|\Scal| |\Acal|}{(1-\gamma)^3\epsilon^2}$ & $\Big(0,\frac1{\sqrt{(1-\gamma)|\Scal|}}\Big] $  & 
			{ \cite{azar2013minimax}}
			\\
			\hline
			Randomized Primal-Dual Method& $ C\frac{|\Scal| |\Acal|}{(1-\gamma)^4\epsilon^2} $  &$(0,(1-\gamma)^{-1}]$& \cite{wang2017randomized}\\
			\hline
			Sublinear Randomized
			Value Iteration & $ \frac{|\Scal| |\Acal|}{(1-\gamma)^4 \epsilon^2}\cdot \poly\log\epsilon^{-1}$ &$(0,1]$& \cite{sidford2018variance}
			\\
			\hline
			& &
			\\[-1em]
			Variance Reduced QVI & $ \frac{|\Scal| |\Acal|}{(1-\gamma)^3 \epsilon^2}\cdot \poly\log\epsilon^{-1}$ & $(0,1]$  &\cite{NIPS2018_7765}
			\\[-1em]
			& &
			\\
			\hline
			& &
			\\[-1em]
			Empirical MDP + \emph{any} accurate black-box planner & $ \frac{|\Scal| |\Acal|}{(1-\gamma)^3 \epsilon^2}$ & $(0, {(1-\gamma)^{-1/2}}]$  & This work
			\\[-1em]
			& &
			\\
			\hline
		\end{tabular}
	\end{center}
	\caption{
		\small
		\textbf{Sample Complexity to Compute $\epsilon$-Optimal Policies Using the Generative Sampling Model}: Here $|\Scal|$ is the number of states, $|\Acal|$ is the number of actions per state, $\gamma \in(0,1)$ is the discount factor, and $C$ is an upper bound on the ergodicity. We ignore $\poly\log(|\cS||\cA|/\delta/(1-\gamma))$ factors in the sample complexity. Rewards are bounded between 0 and 1.
		\label{table-sample}}
	\label{tab:literature_runtime_apx}
\end{table*}

\section{Setting}

\paragraph{Markov Decision Process}
We denote a discounted Markov decision process (MDP) as a tuple
$M = (\cS, \cA, P_M, r_M, \gamma)$, where $\cS$ is a finite set of
states, $\cA$ is a finite set of actions,
$P_M:\cS\times \cA\rightarrow\RR^{\cS}$ is the transition kernel (that
is, $P_M(s'~|~s,a)$ is the probability of obtaining state $s'$ when we
take action $a$ in state $s$), $r_M:\cS\times \cA \rightarrow [0,1]$
is the reward function\footnote{We consider the setting where the
  rewards are in $[0,1]$. Our results can be generalized to other
  ranges of reward function via a standard reduction (see
  e.g. \cite{NIPS2018_7765})}, and $\gamma\in(0,1)$ is a discount
factor.  For any $(s,a)$, we denote $P_M(\cdot~|~s,a)\in \RR^{|\cS|}$ as
the probability vector conditioning on state-action pair $(s,a)$.  A
(deterministic) stationary policy is a map $\pi:\cS\rightarrow \cA$ that maps a
state to an action.  The value function of a policy $\pi$ is a vector
$V_M^{\pi}\in \RR^{|\cS|}$, defined as follows.
\begin{align}
\forall s\in \cS:\quad V_M^{\pi}(s):=\EE\Big[\sum_{i=0}^{\infty}\gamma^{t}r_M(s^{t},a^t)|s^0 = s\Big],
\end{align}
where $a^t=\pi(s^t)$ and $s^1, s^2, s^3, \ldots $ are generated from
the distribution $s^{t+1}\sim P_M(\cdot|s^t,a^t)$. We also define an
action value function $Q_M^{\pi}\in \RR^{\cS\times \cA}$ for policy
$\pi$:
\[
\forall (s,a)\in \cS\times \cA:\quad Q_M^{\pi}(s,a) = r_M(s,a) + \gamma P_M(\cdot|s,a)^\top V^{\pi}.
\]
When the MDP $M$ is clear from the context, we drop the subscript to avoid clutter. The goal of a planning algorithm is to find a stationary policy in the MDP which maximizes the expected reward, denoted by $\pi^\star$. The famous theorem of ~\citet{bellman1959functional} shows that there exists a policy $\pi^\star$ which simultaneously maximizes $V^\pi(s_0)$ for all $s_0 \in \cS$. We also use $Q^\star$ and $V^\star$ to denote the value functions induced by $\pi^\star$
We call a policy, $\pi$, $\epsilon$-optimal, if $V^{\pi}(s)\ge V^*(s) -\epsilon$ for all $s\in \cS$.

\paragraph{Generative Model}
Assume we have a access to a \emph{generative model} or a \emph{sampler}, which can provide us with samples $s'\sim P(\cdot~|~s, a)$. Suppose we call our sampler $N$ times at each state action pair. Let $\widehat P$ be our empirical model, defined as follows:
\[
\widehat P(s'|s,a) =  \frac{\textrm{count}(s',s,a)}{N}
\]
where $\textrm{count}(s',s,a)$ is the number of times the state-action pair $(s,a)$ transitions to state $s'$.
We define $\widehat M$ to be the empirical MDP that is identical
to the original $M$, except that it  uses $\widehat P$ instead of
$P$ for the transition kernel. We let $\widehat V^\pi$ and $\widehat Q^\pi$ to denote the value functions of a policy $\pi$ in $\widehat M$, and $\widehat \pi^\star$, $\widehat Q^\star$ and $\widehat V^\star$ refer to the optimal policy and its value functions in $\widehat M$. The reward function $r$ is assumed to be known and deterministic\footnote{If $r$ is unknown, we can use additional $|\cS||\cA|$ samples to obtain the exact value of $r$. If $r$ is stochastic, we can query $|\cS||\cA|/\epsilon^2/(1-\gamma)^{2}$ samples to obtain a sufficiently accurate estimate of its mean. In both cases, the complexity contributed by $r$ is only a lower order term to the present case. We can therefore assume, without loss of generality, $r$ is known and deterministic.}, and hence is identical in $M$ and $\widehat M$.

\paragraph{Optimization Oracle} Our goal in this paper is to determine
the smallest sample size $N$, such that a planner run in $\widehat M$
returns a near-optimal policy in $M$. In order to decouple the
statistical and computational aspects of planning with respect to an
approximate model $\Mhat$, we will make use of an
\emph{optimization oracle} which takes as input an MDP $M$
and returns a policy $\pi$ satisfying:  $\norm{V_{M}^\pi - V_{M}^\star}_\infty \leq
\epsopt$ and $\norm{Q_{M}^\pi - Q_{M}^\star}_\infty \leq
\epsopt$.

We will use this optimization oracle for the empirical
MDP $\Mhat$, and analyze the performance of the returned policy in the
original MDP $M$.  Classical algorithms such as value or policy
iteration~\citep{puterman2014markov} are the most common examples, though
we discuss more sophisticated oracles as well in the next
section.

\section{Main results}
In this section we present our main results. Before presenting our
main theorem, we review some of the key challenges and our
approach. Our high-level approach is to invoke any reasonable
optimization oracle for the sample-based MDP $\Mhat$, and understand
the sub-optimality of the returned policy $\pi$ in the original MDP
$M$. The key challenge is that $\pi$ depends on the randomness in
$\Mhat$, and hence, its value estimate from $\Mhat$ is not an unbiased
estimator of its value in $M$. A usual way to address such issues is
via uniform convergence, that is, we first establish that the values
of all policies are similar in $\Mhat$ and $M$. This then implies that
the high value of $\pi$ in $\Mhat$ translates to a high value in
$M$. Unfortunately, a na\"ive application of this argument yields
bounds scaling as $|\Sset|^2$. \citet{azar2013minimax} do establish
uniform convergence, but use a more careful argument which yields a
bound scaling linearly in $|\Sset|$, but only when the desired
accuracy $\epsilon \leq \sqrt{1/((1-\gamma)|\Sset|)}$, where the
$|\Sset|$ factor in the condition of $\epsilon$ is due to uniform
convergence. \citet{NIPS2018_7765, sidford2018variance} instead use a
more complex algorithmic modification using variance reduction to get
a sharper uniform convergence over a smaller class of policies with
small variance in their value functions. In our result, we instead
rely on a novel technique to directly establish uniform convergence of
our value estimates, while utilizing the most natural algorithmic
scheme of running a black-box optimization oracle on the sample-based
MDP $\widehat M$. We will show the following result for this scheme.

\begin{theorem}
  Suppose $\delta >0$ and  $\epsilon \in (0,(1-\gamma)^{-1/2}]$.
  Let $\pihat$ be any $\epsopt$-optimal policy for  $\wh{M}$,
  i.e.
  \mbox{$\|\wh{Q}^{\pihat} - \wh{Q}^\star\|_{\infty} \le \epsopt $}.
  If
  \[
    N \geq
    \frac{c \gamma \log\big(c \, |\Scal||\Acal|(1-\gamma)^{-1}\delta^{-1}\big)}
    {(1-\gamma)^3\epsilon^2}, 
  \]
  we have
  \[ Q^{ \pihat} \geq  Q^\star - \epsilon -
  \frac{9 \epsopt}{(1-\gamma)}
  \quad\text{and}\quad
  V^{ \pihat} \geq  V^\star - \epsilon -
  \frac{9 \epsopt}{(1-\gamma)},
  \]
  with probability at least $1-\delta$, where $c$ is an absolute constant, provided $\gamma \ge 1/2$.
\label{thm:sample_complexity}
\end{theorem}

Thus, the theorem shows that if $\epsopt$ is made suitably small
(roughly $(1-\gamma)\epsilon$), then we will find an $O(\epsilon)$
sub-optimal policy with
$O\left(\log\frac{|\Scal||\Acal|}{(1-\gamma)\delta}/(1-\gamma)^3/\epsilon^2\right)$
samples in each $s,a$ pair. The total number of samples from the
generative model then is $|\Scal||\Acal| N$ which amounts to
$O\left(|\Scal||\Acal|\log\frac{|\Scal||\Acal|}{(1-\gamma)\delta}/(1-\gamma)^3/\epsilon^2\right)$
samples. As remarked before, this is known to be unimprovable (up to a logarithmic factor) in the
regime $\epsilon\in(0,1]$ due to the lower bounds
of~\citep{azar2012sample, NIPS2018_7765}.

We have so far focused on the statistical aspects of our estimators, since the use of a black-box optimization method in $\Mhat$ allows us to leverage the best possible solutions available. We now discuss some specific implications on the computational complexity of sparse model-based planning, instantiating the bound for some of the natural methods that may be used. Throughout we focus on attaining $\epsopt = O((1-\gamma)\epsilon)$, since that equates the statistical and optimization errors. A very natural idea is to use value iteration (see e.g. \cite{puterman2014markov}), which requires $O[(1-\gamma)^{-1}\cdot \log \epsopt^{-1}]$ iterations, with each iteration taking $O(|\cS||\cA|N)$ time. Thus the overall running time for this algorithm is \[
O\big(|\cS||\cA|N\cdot (1-\gamma)^{-1}\cdot \log \epsopt^{-1}\big)
 = O\bigg(\frac{|\cS||\cA|\cdot \log\frac{|\cS||\cA|}{(1-\gamma)\delta}\cdot \log \frac{1}{(1-\gamma)\epsilon}}{(1-\gamma)^4\epsilon^2}\bigg).
\]
Policy iteration methods (see again \cite{puterman2014markov}) can obtain an $\epsopt$-optimal policy within the same iteration complexity bound as value iteration.
However, each iteration of the policy iteration requires solving a linear system of size $|\cS|^2$, which can be expensive.
This computation time can be additionally improved.
For instance, after initial phase of reading $O(|\cS||\cA|N)$ data points, \cite{sidford2018variance} give a randomized algorithm to obtain an $\epsopt$-optimal policy with probability at least $1-\delta$  in time
\[
\wt{O}\bigg[\bigg(\mathrm{nnz}(\wh{P}) + \frac{|\cS||\cA|}{(1-\gamma)^3}\bigg)\cdot \log\Big(\frac{1}{\epsopt}\Big)\cdot \log\frac{1}{\delta}\bigg]
=\wt{O}\bigg[\frac{|\cS||\cA|}{(1-\gamma)^3}\cdot\bigg(\frac{ \log\frac{|\cS||\cA|}{(1-\gamma)\delta}}{\min(\epsilon^2, 1)}\bigg)\cdot \log\Big(\frac{1}{(1-\gamma)\epsilon}\Big)\cdot \log\frac{1}{\delta}\bigg],
\]
where $\wt{O}$ hides $\mathrm{poly}\log\log$ factors and $\mathrm{nnz}(P)$ means the number of non-zero entries in $P$. Thus, the computational complexity of this scheme is nearly-linear in the total sample size up to additional logarithmic factors.
There are other results for obtaining an exactly optimal policy for the MDP $\wh{M}$ as well, for instance the SIMPLEX policy iteration \cite{ye2011simplex}, which runs in time $O(\mathrm{poly}(|\cS||\cA|N/(1-\gamma))$.

\section{Analysis}

We begin with some notation needed for our analysis, and then give a high-level outline of the proof, along with some basic lemmas. We then present our main technical novelty, which is a construction of an auxiliary MDP as a device to guarantee uniform convergence of value functions. We conclude by providing the proof of the theorem in terms of the key lemmas, deferring the proofs of the lemmas to the appendix.

\paragraph{Additional Notation}
For a vector $v$, we let $(v)^2$, $\sqrt{v}$, and $|v|$ be the
component-wise square, square root, and absolute value operations.
We let $\mathds{1}$ denotes the vector of all ones (adapting to dimensions based on the context).
It is helpful to overload notation and let $P$ be a matrix of
	size $(\Scal\times \Acal)\times \Scal$  where the entry
$P_{(s,a),s'}$ is equal to $P(s'~|~s,a)$. Also, let $P_{s,a}$ denote the vector $P(\cdot~|~s,a)$.
	We also define $P^\pi$ to be the transition matrix on
	state-action pairs induced by a deterministic policy $\pi$. In
	particular,
	\begin{align*}
		P^\pi_{(s,a),(s',a')} = P(s'~|~s,a)\quad\text{if}\quad a'=\pi(s'),\quad \text{and}\quad
		P^\pi_{(s,a),(s',a')} = 0\quad\text{if}\quad a'\neq\pi(s').
	\end{align*}
	 With this notation, we have
	\begin{align*}
	Q^\pi  = r+\gamma PV^\pi= r+\gamma P^\pi Q^\pi, \quad\mbox{and}\quad
	Q^\pi = (I-\gamma P^\pi)^{-1} r.
	\end{align*}
	Slightly abusing the notation, for $V\in \R^{\Scal}$, we define the vector $\mathrm{Var}_{P}(V)\in \R^{\Scal\times\Acal}$ as:
	\[
	\mathrm{Var}_{P}(V)(s,a) := \mathrm{Var}_{P(\cdot|s,a)}(V), \quad \mbox{so that}\quad \mathrm{Var}_{P}(V) = P (V)^2 - (P V)^2,
	\]
    where the squares are applied componentwise.
	We also define  $\Sigma^\pi_M$ as the variance of the discounted reward, i.e.
	\[
	\Sigma^\pi_M(s,a) := \E\bigg[\bigg(\sum_{t=0}^\infty \gamma^t r(s_t,a_t)-Q^\pi_M(s,a)\bigg)^2~
	\bigg|~s_0=s,a_0=a\bigg]
\]
where the expectation is induced under the trajectories induced by
$\pi$ in $M$. It can be verified that, for all $\pi$, $\Sigma^\pi$ satisfies
the following Bellman style, self-consistency conditions (see Lemma 6 in
\cite{azar2013minimax}):
	\begin{equation}\label{eq:bellman_var}
	\Sigma^\pi_M = \gamma^2 \mathrm{Var}_{P}(V^{\pi}_M)+\gamma^2 P^\pi \Sigma^\pi_M
	\end{equation}
	It is straightforward to verify that $\|\Sigma^\pi_M\|_\infty\leq
	\gamma^2/(1-\gamma)^2$.


\subsection{Errors in empirical estimates}

We begin the analysis by stating some basic results about empirical
estimates of values derived from $\Mhat$ relative to their true values
in $M$. We start with stating a lemma on componentwise error bounds. Its proof has been postponed to the appendix.
\begin{lemma}[Componentwise bounds]\label{lemma:bounds}
	For any policy $\pi$, we have
	\[
	Q^\pi - \widehat Q^\pi =  \gamma (I-\gamma P^\pi)^{-1}( P- \widehat P) \widehat V^\pi.
	\]
In addition, we have:
  \[
  Q^{\pi}\geq  Q^\star
-\| Q^{\pi} - \widehat Q^\pi\|_\infty - \| \widehat Q^\pi - \widehat Q^\star \|_\infty - \|\widehat Q^{\pi^\star}- Q^\star\|_\infty
    \]
and
  \[
V^{\pi}\geq  V^\star
-\| Q^{\pi} - \widehat Q^\pi\|_\infty - \| \widehat V^\pi - \widehat V^\star \|_\infty - \|\widehat Q^{\pi^\star}- Q^\star\|_\infty
\]
\end{lemma}

\begin{proof}
	For any policy $\pi$,
	\begin{eqnarray*}
		Q^\pi - \widehat Q^\pi &= & (I-\gamma P^\pi)^{-1} r -
		(I-\gamma\widehat P^\pi)^{-1} r\\
		&= & (I-\gamma P^\pi)^{-1}((I-\gamma \widehat P^\pi)
		-(I-\gamma P^\pi)) \widehat Q^\pi\\
		&= & \gamma (I-\gamma P^\pi)^{-1}( P^\pi- \widehat P^\pi)\widehat Q^\pi\\
		&= & \gamma (I-\gamma P^\pi)^{-1}( P- \widehat P) \widehat V^\pi\, .
	\end{eqnarray*}
	For the second claim,
	\begin{align*}
	Q^{\pi} - Q^\star &= Q^{\pi} - \widehat Q^\star+ \widehat Q^\star- Q^\star\\
	&\geq
	Q^{\pi} - \widehat Q^\star+ \widehat Q^{\pi^\star}- Q^\star
	\\
	&\geq
	-\| Q^{\pi} - \widehat Q^\star\|_\infty- \|\widehat Q^{\pi^\star}- Q^\star\|_\infty\\
	&\geq -\| Q^{\pi} - \widehat Q^\pi\|_\infty - \| \widehat Q^\pi - \widehat Q^\star \|_\infty - \|\widehat Q^{\pi^\star}- Q^\star\|_\infty.
	\end{align*}
	Similarly, we have
	 \[
	V^{\pi}\geq  V^\star
	-\| V^{\pi} - \widehat V^\pi\|_\infty - \| \widehat V^\pi - \widehat V^\star \|_\infty - \|\widehat V^{\pi^\star}- V^\star\|_\infty.
	\]
	Another application of 
	\[\forall \pi,\quad \|V^\pi - \widehat{V}^{\pi}\|_{\infty} \le \|Q^\pi - \widehat{Q}^{\pi}\|_{\infty}\]  completes the proof of the first inequality.
\end{proof}

We hope to invoke the second part of the lemma to establish
Theorem~\ref{thm:sample_complexity}, where the middle term is
the optimization error, and we will focus on bounding the other two
terms. We next state another basic lemma.

\begin{lemma}
For any policy $\pi$, MDP $M$ and vector $v \in \R^{|\Sset|\times|\Aset|}$, we have $\norm{(I-\gamma P^\pi)^{-1} v}_\infty \leq \norm{v}_\infty/(1-\gamma)$.
\label{lemma:horizon}
\end{lemma}

\begin{proof}
	Note that $v=(I - \gamma P^\pi)(I - \gamma P^\pi)^{-1}  v = (I - \gamma P^\pi) w$, where
	$w = (I - \gamma P^\pi)^{-1}  v$.
	By triangle inequality, we have
	\begin{align*}
	\|v\|= \norm{(I - \gamma P^\pi) w} &\geq \norm{w}_\infty - \gamma \norm{P^\pi w}_\infty \geq \norm{w}_\infty - \gamma \norm{w}_\infty,
	\end{align*}
	where the final inequality follows since $P^\pi w$ is an average of the elements of $w$ by the definition of $P^\pi$ so that $\norm{P^\pi w}_\infty \leq \norm{w}_\infty$. Rearranging terms completes the proof.
\end{proof}

Our next lemma is a key observation in Lemma 6
of~\citet{azar2012sample}, namely the Bellman property of a policy's
variance and its accumulation under the transition operator of the
corresponding policy. We provide a short proof in the appendix for completeness.

\begin{lemma}\label{lemma:var_add}
For any policy $\pi$ and MDP $M$,
\[
\Big\| (I-\gamma P^\pi)^{-1}\sqrt{\mathrm{Var}_{P}(V^{\pi}_M)}\Big\|_\infty
\leq \sqrt{\frac{2}{(1-\gamma)^3} },
\]
where $P$ is the transition model of $M$.
\end{lemma}

\begin{proof}
	Note that $(1-\gamma)(I-\gamma P^\pi)^{-1}$ is matrix whose rows
	are a probability distribution. For a positive vector $v$ and
	a distribution $\nu$ (where $\nu$ is vector of the same dimension of
	$v$),  Jensen's inequality implies that $\nu \cdot\sqrt{ v}\leq
	\sqrt{\nu \cdot v}$. This implies:
	\begin{eqnarray*}
		\| (I-\gamma P^\pi)^{-1}\sqrt{v}\|_\infty &=&
		\frac{1 }{1-\gamma}\|(1-\gamma) (I-\gamma P^\pi)^{-1}\sqrt{v}\|_\infty\\
		&\leq& \sqrt{\Big\|\frac{1}{1-\gamma} (I-\gamma P^\pi)^{-1} v\Big\|_\infty}\\
		&\leq& \sqrt{\Big\|\frac{2 }{1-\gamma} (I-\gamma^2 P^\pi)^{-1}
			v\Big\|_\infty}\, .
	\end{eqnarray*}
	where we have used that $\| (I-\gamma P^\pi)^{-1} v\|_\infty\leq 2\|
	(I-\gamma^2 P^\pi)^{-1} v\|_\infty$ (which we will prove shortly).
	The
	proof is completed as follows: by Equation~\ref{eq:bellman_var}, $\Sigma_M^\pi= \gamma^2(I-\gamma^2
	P^\pi)^{-1}\mathrm{Var}_{P}(V^{\pi}_M) $, so taking
	$v=\mathrm{Var}_{P}(V^{\pi}_M)$ and using that
	$\|\Sigma^\pi_M\|_\infty\leq \gamma^2/(1-\gamma)^2$ completes the proof.
	
	Finally, to see that $\| (I-\gamma P^\pi)^{-1} v\|_\infty\leq 2\|
	(I-\gamma^2 P^\pi)^{-1} v\|_\infty$,
	observe:
	\begin{eqnarray*}
		\| (I-\gamma P^\pi)^{-1} v\|_\infty &=&
		\| (I-\gamma P^\pi)^{-1}(I-\gamma^2P^\pi) (I-\gamma^2P^\pi)^{-1} v\|_\infty \\
		&=& \| (I-\gamma P^\pi)^{-1}\Big((1-\gamma)I +\gamma(I-\gamma P^\pi)\Big) (I-\gamma^2P^\pi)^{-1} v\|_\infty \\
		&=&\| \Big( (1-\gamma) (I-\gamma P^\pi)^{-1} +\gamma I \Big) (I-\gamma^2P^\pi)^{-1} v\|_\infty \\
		&\leq&(1-\gamma) \|(I-\gamma P^\pi)^{-1}(I-\gamma^2P^\pi)^{-1} v\|_\infty +\gamma \|(I-\gamma^2P^\pi)^{-1} v\|_\infty \\
		&\leq& \frac{1-\gamma }{1-\gamma}\|(I-\gamma^2P^\pi)^{-1} v\|_\infty+\gamma\|(I-\gamma^2P^\pi)^{-1} v\|_\infty \\
		&\leq& 2 \|(I-\gamma^2P^\pi)^{-1} v\|_\infty
	\end{eqnarray*}
	which proves the claim.
\end{proof}

Finally, it will be useful to also have more direct bounds on the errors in our value estimates which follow directly from Hoeffding's inequality, even though we are eventually after more careful bounds that account for variance. This result can be also be found as Lemma 4 in~\citet{azar2013minimax}, and is a standard concentration argument.
For completeness, we provide its proof here. 
\begin{lemma}[Crude Value Bounds, Lemma 4 in~\citet{azar2013minimax}]\label{lemma:crude}
  Let $\delta\geq 0$. With probability greater than $1-\delta$,
  \begin{eqnarray*}
    \|Q^\star-\widehat Q^{\pi^\star}\|_\infty \leq\Delta_{\delta,N} \quad\mbox{and}\quad
    \|Q^\star-\widehat Q^\star\|_\infty \leq
\Delta_{\delta,N}\,,~~\mbox{where}~~\Delta_{\delta,N} := \frac{\gamma}{(1-\gamma)^2} \sqrt{\frac{2\log(2|\Scal||\Acal|/\delta)}{N}}.
  \end{eqnarray*}
\end{lemma}

\begin{proof}
	Note that $V^{\star}$ is a fixed vector independent with the randomness in $\wh{P}$.
	Moreover, $\|V^{\star}\|_{\infty}\le (1-\gamma)^{-2}$.
	Thus, by Hoeffding bound and a union bound over all $\cS\times \cA$, we have, with probability at least $1-\delta$,
	\begin{align*}
	\|
	( \wh{P}-  P)  V^{\star}
	\|_{\infty}\le \sqrt{\frac{2\log(2|\cS||\cA|/\delta)}{N\cdot (1-\gamma)^2}}.
	\end{align*}
	For the rest of the proof, we condition on the event that the above inequality holds.

	Next we show the first inequality.
	Note that for any $\pi$, we have,
	\begin{eqnarray*}
		Q^\pi - \widehat Q^\pi & = & (I-\gamma P^\pi)^{-1} r -
		(I-\gamma\widehat P^\pi)^{-1} r\\
		&= & (I-\gamma \wh{P}^\pi)^{-1}((I-\gamma  \wh{P}^\pi)
		-(I-\gamma {P}^\pi))  Q^\pi\\
		&= & \gamma (I-\gamma \wh{P}^\pi)^{-1}( {P}^\pi-  \wh{P}^\pi)Q^\pi\\
		&= & \gamma (I-\gamma \wh{P}^\pi)^{-1}( {P}-  \wh{P})  V^\pi\, .
	\end{eqnarray*}
	Consider $\pi^{\star}$.
	Since $(I-\gamma \wh{P}^\pi)^{-1}
	= \sum_{i=0}\gamma^i (\wh{P}^\pi)^i$ and $(\wh{P}^\pi)^i$ is a probability matrix, we have
	\begin{align*}
	\|\gamma (I-\gamma \wh{P}^\pi)^{-1}( \wh{P}-  P)  V^{\star}\|_{\infty}
	&\le \gamma \sum_{i=0}^{\infty}\|\gamma^i (\wh{P}^\pi)^i( \wh{P}-  P)  V^{\star} \|_{\infty}
	\le \gamma \sum_{i=0}^{\infty}\|\gamma^i( \wh{P}-  P)  V^{\star} \|_{\infty}\\
	&\le \frac{\gamma}{(1-\gamma)}\cdot \sqrt{\frac{2\log(2|\cS||\cA|/\delta)}{N\cdot (1-\gamma)^2}}
	\end{align*}
	as desired.
	
	Now we consider the second inequality.
	Let ${\cT}$ be the Bellman optimality operator on ${M}$, i.e., for any $V\in \RR^{\cS}$
	\begin{align*}
	\forall s\in \cS:\quad {\cT}(V)(s) &=
	\max_{a}\big[r(s, a) + {P}(\cdot~|~s,a)^\top V\big],\quad
	\text{and}\\ \forall (s,a)\in \cS\times\cA:\quad {\cT}(Q)(s,a) &=
	r(s, a) + \sum_{s'}P(s'~|~s,a) \max_{a'} Q(s',a').
	\end{align*}
	Let $\wh{\cT}$ be the Bellman optimality operator on $\wh{M}$. Further recalling our notations $P^\pi$ and $\widehat P^\pi$, we have
	
	\begin{align*}
	\|Q^{\star} - \wh{Q}^{\star}\|_{\infty} &= \|\cT Q^{\star} - \wh{\cT}\wh{Q}^{\star}\|_{\infty}\\
	&\le \|\cT Q^{\star} - r - \wh{P}^{\pi^{\star}}{Q}^{\star}\|_{\infty}
	+ \|\wh{P}^{\pi^{\star}}{Q}^{\star} + r -\wh{\cT}\wh{Q}^{\star}\|_{\infty} \\
	&= \gamma\|P^{\pi^{\star}}Q^{\star} - \wh{P}^{\pi^{\star}}Q^{\star}\|_{\infty} +  \gamma \|\wh{P}^{\pi^{\star}}Q^{\star} - \wh{P}^{\wh{\pi}^{\star}}\wh{Q}^{\star}\|_{\infty}\\
	&=
	\gamma  \|(P-\wh{P})V^{\star}\|_{\infty} + \gamma \|\wh{P}V^{\star} - \wh{P}\wh{V}^{\star}\|_{\infty}\\
	&\le \gamma  \|(P-\wh{P})V^{\star}\|_{\infty}  +
	\gamma \|V^{\star} - \wh{V}^{\star}\|_{\infty}\\
	&\le \gamma  \|(P-\wh{P})V^{\star}\|_{\infty}  +
	\gamma \|Q^{\star} - \wh{Q}^{\star}\|_{\infty}.
	\end{align*}
	Solving for $\|Q^{\star} - \wh{Q}^{\star}\|$, we complete the proof.
\end{proof}

We observe that these simple bounds are worse than what Theorem~\ref{thm:sample_complexity} posits by a factor of $\sqrt{1/(1-\gamma)}$, and removing this additional factor requires a significantly more careful analysis as we will see in the remainder of this section.

\subsection{An $s$-absorbing MDP M}

In order to improve upon the crude bounds in Lemma~\ref{lemma:crude}, we would like to directly bound the errors in our value estimates using the componentwise bounds of Lemma~\ref{lemma:bounds}. Doing so requires an understanding of quantities such as $|( P- \widehat P) \widehat
V^\star|$ and $|( P- \widehat P)
\widehat V^{\pi^\star}|$, which we will do next. However
$\widehat V^\star$ and $\widehat V^{\pi^\star}$ depend on
$\widehat P$, so that we are not able to directly apply a standard
concentration argument. We now address this challenge by
providing a method to decouple these dependencies.

For a state $s$ and a scalar $u$, define the MDP
$M_{s,u}$ as follows: $M_{s,u}$ is identical
to $M$ except that state $s$ is absorbing in $M_{s,u}$, i.e. $P_{M_{s,u} }(s|s,a)=1$ for all $a$,
and the instantaneous reward at state $s$ in $M_{s,u}$ is $(1-\gamma)
u$; the remainder of the transition model and reward function are
identical to those in $M$. In order to avoid notational clutter, we use $V^\pi_{s,u}$ to denote the value function $V^\pi_{M_{s,u}}$ and correspondingly for $Q$ and reward and transition functions. This implies that for
all policies $\pi$:
\[
V^\pi_{s,u}(s) = u,
\]
since $s$ is absorbing with instantaneous reward $(1-\gamma) u$.

For some state $s$, we will only consider $M_{s,u}$ for $u$ in a
finite set $U_s$, where
\[
  U_s\subset [V^\star(s)-\Delta_{\delta,N}\,
  V^\star(s)+\Delta_{\delta,N} ] \, .
  \]
In particular, we will set  $U_s$ to consist of evenly spaced elements in this
interval, where we set the size of $|U_s|$ appropriately later on.
As before, we let
$\widehat M_{s,u}$ denote the MDP that uses the empirical model $\widehat P$ instead of
$P$, at all non-absorbing states and abbreviate the value functions in $\widehat M_{s,u}$ as $\widehat V^\pi_{s,u}$.

\begin{lemma}\label{lemma:union_u}
Fix a state $s$, an action $a$, a finite set $U_s$, and $\delta\geq 0$. With probability greater than $1-\delta$,
it holds that for all  $u\in U_s$,
\begin{eqnarray*}
|( P_{s,a}- \widehat P_{s,a})
  \cdot \widehat V_{s,u}^\star | &\leq &
\sqrt{\frac{ 2\log(4|U_s|/\delta)}{N}} \sqrt{\mathrm{Var}_{P_{s,a}}(\widehat V_{s,u}^\star)}
+\frac{2\log(4|U_s|/\delta)}{(1-\gamma) 3N}\\
|( P_{s,a}- \widehat P_{s,a})
  \cdot \widehat V_{ s,u}^{\pi^\star} | &\leq &
\sqrt{\frac{ 2\log(4|U_s|/\delta)}{N}} \sqrt{\mathrm{Var}_{P_{s,a}}(\widehat V_{s,u}^{\pi^\star})}
+\frac{2\log(4|U_s|/\delta)}{(1-\gamma) 3N}
\end{eqnarray*}
\end{lemma}
\begin{proof}
The random variables $\widehat P_{s,a}$\footnote{Note that $P_{s,a}$ and $\widehat P_{s,a}$ are from the original MDPs $M$ and $\widehat M$ and not the absorbing versions, as the latter induce degenerate transitions in $s$ for all actions $a$.} and $\widehat V_{s,u}^\star$ are
independent. The result now follows from Bernstein's inequality along with a
union bound over all $U_s$.
\end{proof}

This independence of $\widehat P_{s,a}$ from the value function $\widehat V^\star_{s,u}$ is the biggest upshot of our construction. Note that a similar statement does not hold for $\widehat V^\star$. We next need to understand how to construct $U_s$ so that $\widehat V^\star_{s,u}$ provides a good approximation for $\widehat V^\star$, for some $u \in U_s$. The following two lemmas provide helpful properties of these absorbing
state MDPs to build towards this goal.

\begin{lemma}\label{lemma:cover1}
Let $u^* = V^\star_M(s)$ and $u^\pi = V^\pi_M(s)$. We have
\[
V^\star_M = V^\star_{s,u^\star} \,,\quad\mbox{and for all policies $\pi$},\quad
V^\pi_M = V^\pi_{M_{s,u^\pi}} \, .
\]
\end{lemma}

\begin{proof}
	To prove the first claim, it suffices to verify that $V^\star_M$
	satisfies the Bellman optimality conditions in
	$M_{s,u^\star}$. To see this, observe that at state $s$, the
	Bellman equations are trivially satisfied as $s$ is absorbing with
	value $u^\star = V^\star_M(s)$ at state $s$ by construction. For state $s'\neq s$,  the outgoing transition model at $s'$ in
	$M_{s,u^\star}$ is identical to that in $M$. Since $V^\star_M$
	satisfies the Bellman optimality conditions at state $s'$ in $M$, it
	must also satisfy Bellman optimality conditions at state $s'$ in $M$. The proof of
	the second claim is analogous.
\end{proof}

This lemma gives a good setting for $u$, but we also need robustness to misspecification of $u$ as we seek to construct a cover. The next lemma provides this result.

\begin{lemma}\label{lemma:cover2}
For all states $s$, $u,u' \in \R$, and policies $\pi$,
\begin{eqnarray*}
\|Q^\star_{s,u} - Q^\star_{s,u'}\|_\infty \leq |u-u'|\quad \mbox{and}\quad
\|Q^\pi_{s,u} - Q^\pi_{s,u'}\|_\infty \leq |u-u'|\,.
\end{eqnarray*}
\end{lemma}

\begin{proof}
	First observe
	\[
	\|r_{s,u} - r_{s,u'}\|_\infty=(1-\gamma)|u-u'|,
	\]
	since these two reward functions differ only in state $s$, in which case
	$r_{s,u}(s,a) =(1-\gamma)u$ and  $r_{s,u'}(s,a)=(1-\gamma)u'$.
	Let $\pi_{s,u}$ be the optimal policy in $M_{s,u}$. Note
	\begin{eqnarray*}
		Q^\star_{s,u} -Q^\star_{s,u'}
		&=&  Q^\star_{s,u} -\max_\pi (I-\gamma P_{s,u'}^\pi)^{-1} r_{s,u'}
		\leq  Q^\star_{s,u} - (I-\gamma P_{s,u'}^{\pi_{s,u}})^{-1} r_{s,u'}\\
		&\stackrel{(a)}{=}&  (I-\gamma P_{s,u'}^{\pi_{s,u}})^{-1} (r_{s,u}-r_{s,u'})
		\leq  \frac{1}{1-\gamma}\| r_{s,u'}- r_{s,u}\|_\infty=  |u-u'| ,
	\end{eqnarray*}
	where the equality $(a)$ follows since $P_{s,u}$ only depends on the state $s$ and not the value $u$. The proof of the lower bound is analogous, which completes the proof
	of the first claim. The proof of the second claim can be obtained with
	a similar argument.
\end{proof}

With these two lemmas, we now show the main result of this section.

\begin{proposition}\label{prop:cover3}
Fix a state $s$, an action $a$, a finite set $U_s$, and $\delta\geq 0$.   With probability greater than $1-2\delta$,
it holds that for all  $u\in U_s$,
\begin{eqnarray*}
|( P_{s,a}- \widehat P_{s,a})
  \cdot \widehat V^\star | &\leq &
\sqrt{\frac{ 2\log(4|U_s|/\delta)}{N}}
\sqrt{\mathrm{Var}_{P_{s,a}}(\widehat V^\star)}\\
&&+\min_{u\in U_s}|\widehat V^\star(s)-u| \left(1
+\sqrt{\frac{ 2\log(4|U_s|/\delta)}{N}} \right)
+\frac{2\log(4|U_s|/\delta)}{(1-\gamma) 3N}
\end{eqnarray*}
\begin{eqnarray*}
|( P_{s,a}- \widehat P_{s,a})
  \cdot \widehat V^{\pi^\star} | &\leq &
\sqrt{\frac{ 2\log(4|U_s|/\delta)}{N}}
\sqrt{\mathrm{Var}_{P_{s,a}}(\widehat V^{\pi^\star})}\\
&&+\min_{u\in U_s}|\widehat V^{\pi^\star} (s)-u| \left(1
+\sqrt{\frac{ 2\log(4|U_s|/\delta)}{N}} \right)
+\frac{2\log(4|U_s|/\delta)}{(1-\gamma) 3N}
\end{eqnarray*}
\end{proposition}

\begin{proof}
	By Lemma~\ref{lemma:union_u}, with probability greater than
	$1-\delta$, we have that for all $u\in U_s$.
	{\small
	\begin{eqnarray*}
		|( P_{s,a}- \widehat P_{s,a})\cdot \widehat V^\star |
		&=&|( P_{s,a}- \widehat P_{s,a})\cdot (\widehat V^\star-V_{s,u}^\star+V_{s,u}^\star)|\\
		&\leq& |( P_{s,a}- \widehat P_{s,a})\cdot (\widehat V^\star-V_{s,u}^\star)|
		+ |( P_{s,a}- \widehat P_{s,a})\cdot (V_{s,u}^\star)|\\
		&\leq& \|\widehat V^\star-V_{s,u}^\star\|_\infty
		+\sqrt{\frac{ 2\log(4|U_s|/\delta)}{N}} \sqrt{\mathrm{Var}_{P_{s,a}}(V_{s,u}^\star)}
		+\frac{2\log(4|U_s|/\delta)}{(1-\gamma) 3N}\\
		&\leq& \|\widehat V^\star-V_{s,u}^\star\|_\infty
		+\sqrt{\frac{ 2\log(4|U_s|/\delta)}{N}} \sqrt{\mathrm{Var}_{P_{s,a}}(\widehat V^\star -V_{s,u}^\star - \widehat V^\star)}
		+\frac{2\log(4|U_s|/\delta)}{(1-\gamma) 3N}\\
		&\leq&
		\|\widehat V^\star-V_{\widehat M_{s,u}}^\star\|_\infty \left(1
		+\sqrt{\frac{ 2\log(4|U_s|/\delta)}{N}} \right)
		+\sqrt{\frac{ 2\log(4|U_s|/\delta)}{N}} \sqrt{\mathrm{Var}_{P_{s,a}}(\widehat V^\star)}
		+\frac{2\log(4|U_s|/\delta)}{(1-\gamma) 3N}
	\end{eqnarray*}
	}
	using the triangle inequality,
	$\sqrt{\mathrm{Var}_{P_{s,a}}(V_1 + V_2)}\le \sqrt{\mathrm{Var}_{P_{s,a}}(V_1)} + \sqrt{\mathrm{Var}_{P_{s,a}}(V_2)}$.
	
	By Lemmas~\ref{lemma:cover1} and ~\ref{lemma:cover2},
	\[
	\|\widehat V^\star-V_{s,u}^\star\|_\infty=
	\|\widehat V_{s, \widehat V^\star(s)}^\star-V_{s,u}^\star\|_\infty \leq |\widehat V^\star(s)-u|\,.
	\]
	Since the above holds for all $u\in U_s$, we may take the best
	possible choice, which completes the proof of the first claim. The proof of the second claim is analogous.
\end{proof}

The proposition, combined with an accounting of the discretization level yields the following result.

\begin{lemma}\label{lemma:cover4}
With probability greater than $1-\delta$,
{\small
\begin{eqnarray*}
|( P- \widehat P) \widehat V^\star | &\leq &
\sqrt{\frac{ 8\log(8 |\Scal||\Acal|/((1-\gamma)\delta))}{N}}
\sqrt{\mathrm{Var}_{P}(\widehat V^\star)}
+\Delta^\prime_{\delta,N}\mathds{1}\\
|( P- \widehat P)
\widehat V^{\pi^\star} | &\leq &
\sqrt{\frac{ 8\log(8 |\Scal||\Acal|/((1-\gamma)\delta))}{N}}
\sqrt{\mathrm{Var}_{P}(\widehat V^{\pi^\star})}
+\Delta^\prime_{\delta,N}\mathds{1}
\end{eqnarray*}
}
where
\[
\Delta^\prime_{\delta,N}= \sqrt{\frac{c\log(c |\Scal||\Acal|/((1-\gamma)\delta))}{N}}
+\frac{c\log(c |\Scal||\Acal|/((1-\gamma)\delta))}{(1-\gamma) N}
\]
with $c$ being an absolute constant.
\end{lemma}

\begin{proof}
	We take $U_s$ to be the
	evenly spaced elements in the interval
	$[V^\star(s)-\Delta_{\delta/2,N}\,  V^\star(s)+\Delta_{\delta/2,N} ]$,
	and we take the size of $U_s$ to be $|U_s|=\frac{1}{(1-\gamma)^2}$.
	By Lemma~\ref{lemma:crude}, with probability greater than
	$1-\delta/2$, we have $\widehat V^\star(s)\in
	[V^\star(s)-\Delta_{\delta/2,N}\, V^\star(s)+\Delta_{\delta/2,N} ]$
	for all $s$. This implies:
	\[
	\min_{u\in U_s}|\widehat V^\star(s)-u|
	\leq \frac{2\Delta_{\delta/2,N}}{|U_s|-1}
	= \frac{2}{|U_s|-1}\frac{\gamma}{(1-\gamma)^2} \sqrt{\frac{4\log(4|\Scal||\Acal|/\delta)}{N}}
	\leq  4\gamma\sqrt{\frac{4\log(4|\Scal||\Acal|/\delta)}{N}}
	\]
	where we have used that that $\widehat V^\star(s)$ will land in one of $|U_s|-1$ evenly sized
	sub-intervals of length $2\Delta_{\delta/2,N}/(|U_s|-1)$.  Now we use $\delta/(2|\Scal||\Acal|)$, so
	that the claims in Proposition~\ref{prop:cover3} hold with probability greater than
	$1-\delta/2$ for all state action pairs. The first claim follows by substitution and noting that
	probability  of either event failing is less than $\delta/2$. The proof
	of the second claim is analogous; note that Lemma~\ref{lemma:crude} and
	Proposition~\ref{prop:cover3} hold simultaneously with regards to the both
	claims regarding $\pi^\star$ and $\widehat \pi^\star$ so no further
	modifications to the failure probability are required.
\end{proof}

\subsection{The proof of Theorem~\ref{thm:sample_complexity}}
\label{subsection:main_proof}

Theorem~\ref{thm:sample_complexity} immediately follows from the
following lemma combined with Lemma~\ref{lemma:bounds}.

\begin{lemma}\label{lemma:final-bound}
Let $\pihat$ be any policy satisfying the condition of
Theorem~\ref{thm:sample_complexity}. Then we have
{\small
\begin{eqnarray*}
\|Q^{\pihat} - \widehat Q^{\pihat}\|_\infty
&\leq& \frac{\gamma}{1-\alpha_{\delta,N}}
\left(\sqrt{ \frac{c }{(1-\gamma)^3}\frac{\log(c|\Scal||\Acal|/((1-\gamma)\delta))}{N}}
+\frac{c \log(c|\Scal||\Acal|/((1-\gamma)\delta))}{(1-\gamma)^2 N} \right)\\
&&+\frac{1}{1-\alpha_{\delta,N}} \cdot \frac{2\gamma\epsopt}{1-\gamma}
\left(1+\sqrt{\frac{\log(c|\Scal||\Acal|/((1-\gamma)\delta))}{N}} \right)\\
\| Q^\star-\widehat Q^{\pi^\star}\|_\infty
&\leq&\frac{\gamma}{1-\alpha_{\delta,N}}
\left(\sqrt{ \frac{c}{(1-\gamma)^3}\frac{\log(c|\Scal||\Acal|/((1-\gamma)\delta))}{N}}
+\frac{c \log(c|\Scal||\Acal|/((1-\gamma)\delta))}{(1-\gamma)^2 N}\right)\\
\end{eqnarray*}
}
where $c$ is an absolute constant and where:
\[
\alpha_{\delta,N} = \frac{\gamma}{1-\gamma} \sqrt{\frac{ 8\log(8|\Scal||\Acal|/((1-\gamma)\delta))}{N}}\, .
\]
\end{lemma}

\begin{proof}
We have:
\begin{eqnarray*}
&&\|Q^{\pihat} - \widehat Q^{\pihat}\|_\infty \\
&\stackrel{(a)}{=}&
\gamma  \|(I-\gamma P^{\pihat})^{-1} ( P- \widehat P) \widehat V^{\pihat}\|_\infty\\
&\stackrel{(b)}{\leq}&
\gamma  \|(I-\gamma P^{\pihat})^{-1} ( P- \widehat P) \widehat
V^\star\|_\infty + \gamma \norm{(I-\gamma P^{\pi})^{-1}(P-\wh{P}) (\wh{V}^{\pihat} -
\wh{V}^{\star})}_\infty,\\
&\stackrel{(c)}{\leq}&
\gamma  \|(I-\gamma P^{\pihat})^{-1} ( P- \widehat P) \widehat V^\star\|_\infty
+\frac{2\gamma\epsopt}{1-\gamma}\\
&\stackrel{(d)}{\leq}&
\gamma  \|(I-\gamma P^{\pihat})^{-1} \big|( P- \widehat P) \widehat V^\star\big|\|_\infty
+\frac{2\gamma\epsopt}{1-\gamma} ,
\end{eqnarray*}
where (a) uses Lemma~\ref{lemma:bounds}; (b) is the triangle
inequality; (c)  uses Lemma~\ref{lemma:horizon}; (d) uses that $(I-\gamma P^{\widehat
  \pi^\star})^{-1}$ has all positive entries.

Focusing on the first term, we see that
{\small
\begin{eqnarray*}
&&\|Q^{\pihat} - \widehat Q^{\pihat}\|_\infty  \\
&\stackrel{(c)}{\leq}&
\gamma \sqrt{\frac{ 8\log(8 |\Scal||\Acal|/((1-\gamma)\delta))}{N}} \left\|  (I-\gamma P^{\pihat})^{-1}
\sqrt{\mathrm{Var}_{P}(\widehat V^\star)}\right\|_\infty
+\frac{\gamma \Delta^\prime_{\delta,N}}{1-\gamma}
+\frac{2\gamma\epsopt}{1-\gamma}  \\
&\stackrel{(d)}{\leq}&
\gamma \sqrt{\frac{ 8\log(8 |\Scal||\Acal|/((1-\gamma)\delta))}{N}} \left\|  (I-\gamma P^{\pihat})^{-1}
\left(\sqrt{\var_{P}(V^{\pihat})}
+\sqrt{\var_{P}(V^{\pihat}- \widehat V^{\pihat})}\right) \right\|_\infty
\\
&&\qquad\qquad + \gamma \sqrt{\frac{ 8\log(8 |\Scal||\Acal|/((1-\gamma)\delta))}{N}} \left\|  (I-\gamma P^{\pihat})^{-1} \sqrt{\var_{P}(\widehat V^{\pihat}-\widehat V^\star)}\right\|_\infty
 +\frac{\gamma \Delta^\prime_{\delta,N}}{1-\gamma}
+\frac{2\gamma\epsopt}{1-\gamma}  \\
&\stackrel{(e)}{\leq}&
\gamma \sqrt{\frac{ 8\log(8 |\Scal||\Acal|/((1-\gamma)\delta))}{N}}
\left(\sqrt{\frac{2}{(1-\gamma)^3}}
+\frac{ \sqrt{ \|V^{\pihat}-\widehat V^{\pihat}\|_\infty^2}}{1-\gamma} +
	\frac{\epsopt}{1-\gamma}\right)
+\frac{\gamma \Delta^\prime_{\delta,N}}{1-\gamma}
+\frac{2\gamma\epsopt}{1-\gamma}  \\
&\leq&
\gamma \sqrt{\frac{ 8\log(8 |\Scal||\Acal|/((1-\gamma)\delta))}{N}}
\left(\sqrt{\frac{2}{(1-\gamma)^3}}
+\frac{ \|Q^{\widehat \pi}-\widehat Q^{\widehat\pi}\|_\infty}{1-\gamma} +
	\frac{\epsopt}{1-\gamma}\right)
+\frac{\gamma \Delta^\prime_{\delta,N}}{1-\gamma}
+\frac{2\gamma\epsopt}{1-\gamma}  \\
&=&
\gamma \sqrt{\frac{ 8\log(8 |\Scal||\Acal|/((1-\gamma)\delta))}{N}}
\left(\sqrt{\frac{2}{(1-\gamma)^3}}
+\frac{ \|Q^{\widehat \pi}-\widehat Q^{\widehat\pi}\|_\infty}{1-\gamma}
\right)
+\frac{\gamma \Delta^\prime_{\delta,N}}{1-\gamma}\\
&&+\frac{2\gamma\epsopt}{1-\gamma}
\left(1+\sqrt{\frac{ 8\log(8 |\Scal||\Acal|/((1-\gamma)\delta))}{N}}\right) \, ,
\end{eqnarray*}
}
where the inequality (c) uses Lemma~\ref{lemma:cover4}; (d) uses
$\sqrt{\var_{P}(X+Y)} = \sqrt{\E_{P}[(X + Y - \E_{P}[X + Y])^2]} \leq \sqrt{\var_P(X)} + \sqrt{\var_P(Y)}$, by triangle inequality of norms, using $\sqrt{\E_P[Z^2]}$ as the norm;
(e) uses Lemma~\ref{lemma:var_add}.
Solving for $\|Q^{\widehat \pi}-\widehat Q^{\widehat\pi}\|_\infty$
proves the first claim. The proof of the second claim is analogous.
\end{proof}

Let us now show that Theorem~\ref{thm:sample_complexity} follows from
this Lemma.
From the condition on $\pihat$ in the theorem statement, along with
Lemma~\ref{lemma:bounds}, we have
\[
Q^{\pihat} \geq Q^\star - \|Q^{\pihat} - \widehat Q^{\pihat}\|_\infty
- \epsopt - \|\widehat Q^{\pi^\star} - Q^\star\|_\infty.
\]

The condition on $N$ in Theorem~\ref{thm:sample_complexity} (for an
appropriately
chosen absolute contant) implies that
$\alpha_{\delta,N} = \frac{\gamma}{1-\gamma}
\sqrt{\frac{8\log(8|\Scal||\Acal|/((1-\gamma)\delta))}{N}} <
1/2$. This and Lemma~\ref{lemma:final-bound} implies:
\[
Q^{\pihat} \geq Q^\star -
\, 4\gamma
\left(\sqrt{ \frac{1 }{(1-\gamma)^3}\cdot
	\frac{\log(|\Scal||\Acal|/((1-\gamma)\delta))}{N}}
+\frac{ c\cdot \log(|\Scal||\Acal|/((1-\gamma)\delta))}{(1-\gamma)^2 N} \right)
- \frac{8\gamma\epsopt}{1-\gamma} - \epsopt .
\]
Plugging in the choice of $N$ in Theorem~\ref{thm:sample_complexity}
(where the absolute constant in Theorem~\ref{thm:sample_complexity}
need not be the same as that in Lemma~\ref{lemma:bounds})
completes the proof of the theorem.

\section{Conclusion}

This paper sheds new light on a long-studied basic question in reinforcement learning, which is that of a good approach to planning, given an approximate model of the world. While this is a fundamental question in itself, previous advances have also resulted in improved algorithms for harder questions such as sample-efficient exploration. For instance, the Bellman structure of variances in an MDP, observed in~\citet{azar2013minimax} has subsequently formed a crucial component of minimax optimal exploration algorithms~\citep{azar2017minimax, jin2018q, zanette2019tighter, wainwright2019stochastic}. We hope that the new technical components in our work can be similarly reused in broader contexts in future work, beyond their utility in analyzing sparse, model-based planning.

\subsection*{Acknowledgments}
Sham Kakade thanks Rong Ge for numerous helpful discussions.
We thank Csaba Szepesvari, Kaiqing Zhang, and 
Mohammad Gheshlaghi Azar for helpful discussions and pointing out typos in the initial version of the paper.
S. K. gratefully acknowledges funding from the Washington Research
Foundation for Innovation in Data-intensive Discover, the ONR award
N00014-18-1-2247, and NSF Award CCF-1703574.

\bibliographystyle{apalike}
\bibliography{ref}

\begin{thebibliography}{}

\bibitem[Azar et~al., 2012]{azar2012sample}
Azar, M.~G., Munos, R., and Kappen, B. (2012).
\newblock On the sample complexity of reinforcement learning with a generative
  model.
\newblock {\em arXiv preprint arXiv:1206.6461}.

\bibitem[Azar et~al., 2013]{azar2013minimax}
Azar, M.~G., Munos, R., and Kappen, H.~J. (2013).
\newblock Minimax pac bounds on the sample complexity of reinforcement learning
  with a generative model.
\newblock {\em Machine learning}, 91(3):325--349.

\bibitem[Azar et~al., 2017]{azar2017minimax}
Azar, M.~G., Osband, I., and Munos, R. (2017).
\newblock Minimax regret bounds for reinforcement learning.
\newblock In {\em Proceedings of the 34th International Conference on Machine
  Learning-Volume 70}, pages 263--272. JMLR. org.

\bibitem[Bertsekas, 1976]{bertsekas1976dynamic}
Bertsekas, D.~P. (1976).
\newblock Dynamic programming and stochastic control.

\bibitem[Dann and Brunskill, 2015]{dann2015sample}
Dann, C. and Brunskill, E. (2015).
\newblock Sample complexity of episodic fixed-horizon reinforcement learning.
\newblock In {\em Advances in Neural Information Processing Systems}, pages
  2818--2826.

\bibitem[Jaksch et~al., 2010]{jaksch2010near}
Jaksch, T., Ortner, R., and Auer, P. (2010).
\newblock Near-optimal regret bounds for reinforcement learning.
\newblock {\em Journal of Machine Learning Research}, 11(Apr):1563--1600.

\bibitem[Jin et~al., 2018]{jin2018q}
Jin, C., Allen-Zhu, Z., Bubeck, S., and Jordan, M.~I. (2018).
\newblock Is q-learning provably efficient?
\newblock In {\em Advances in Neural Information Processing Systems}, pages
  4863--4873.

\bibitem[Kakade et~al., 2003]{kakade2003sample}
Kakade, S.~M. et~al. (2003).
\newblock {\em On the sample complexity of reinforcement learning}.
\newblock PhD thesis, University of London London, England.

\bibitem[Kalathil et~al., 2014]{kalathil2014empirical}
Kalathil, D., Borkar, V.~S., and Jain, R. (2014).
\newblock Empirical q-value iteration.
\newblock {\em arXiv preprint arXiv:1412.0180}.

\bibitem[Kearns and Singh, 1999]{kearns1999finite}
Kearns, M.~J. and Singh, S.~P. (1999).
\newblock Finite-sample convergence rates for q-learning and indirect
  algorithms.
\newblock In {\em Advances in neural information processing systems}, pages
  996--1002.

\bibitem[Osband and Van~Roy, 2014]{osband2014model}
Osband, I. and Van~Roy, B. (2014).
\newblock Model-based reinforcement learning and the eluder dimension.
\newblock In {\em Advances in Neural Information Processing Systems}, pages
  1466--1474.

\bibitem[Puterman, 2014]{puterman2014markov}
Puterman, M.~L. (2014).
\newblock {\em {M}arkov decision processes: discrete stochastic dynamic
  programming}.
\newblock John Wiley \& Sons.

\bibitem[Sidford et~al., 2018a]{NIPS2018_7765}
Sidford, A., Wang, M., Wu, X., Yang, L., and Ye, Y. (2018a).
\newblock Near-optimal time and sample complexities for solving markov decision
  processes with a generative model.
\newblock In {\em Advances in Neural Information Processing Systems 31}, pages
  5186--5196. Curran Associates, Inc.

\bibitem[Sidford et~al., 2018b]{sidford2018variance}
Sidford, A., Wang, M., Wu, X., and Ye, Y. (2018b).
\newblock Variance reduced value iteration and faster algorithms for solving
  markov decision processes.
\newblock In {\em Proceedings of the Twenty-Ninth Annual ACM-SIAM Symposium on
  Discrete Algorithms}, pages 770--787. SIAM.

\bibitem[Singh and Yee, 1994]{singh1994upper}
Singh, S. and Yee, R. (1994).
\newblock An upper bound on the loss from approximate optimal-value functions.
\newblock {\em Machine Learning}, 16(3):227--233.

\bibitem[Strehl, 2007]{strehl2007probably}
Strehl, A.~L. (2007).
\newblock {\em Probably approximately correct (PAC) exploration in
  reinforcement learning}.
\newblock PhD thesis, Rutgers University-Graduate School-New Brunswick.

\bibitem[Strehl et~al., 2009]{strehl2009reinforcement}
Strehl, A.~L., Li, L., and Littman, M.~L. (2009).
\newblock Reinforcement learning in finite mdps: Pac analysis.
\newblock {\em Journal of Machine Learning Research}, 10(Nov):2413--2444.

\bibitem[Strehl et~al., 2006]{strehl2006pac}
Strehl, A.~L., Li, L., Wiewiora, E., Langford, J., and Littman, M.~L. (2006).
\newblock Pac model-free reinforcement learning.
\newblock In {\em Proceedings of the 23rd international conference on Machine
  learning}, pages 881--888. ACM.

\bibitem[Wainwright, 2019]{wainwright2019stochastic}
Wainwright, M.~J. (2019).
\newblock Stochastic approximation with cone-contractive operators: Sharp
  $\ell$-infty -bounds for q-learning.
\newblock {\em arXiv preprint arXiv:1905.06265}.

\bibitem[Wang, 2017]{wang2017randomized}
Wang, M. (2017).
\newblock Randomized linear programming solves the discounted {M}arkov decision
  problem in nearly-linear running time.
\newblock {\em arXiv preprint arXiv:1704.01869}.

\bibitem[Ye, 2011]{ye2011simplex}
Ye, Y. (2011).
\newblock The simplex and policy-iteration methods are strongly polynomial for
  the {M}arkov decision problem with a fixed discount rate.
\newblock {\em Mathematics of Operations Research}, 36(4):593--603.

\bibitem[Zanette and Brunskill, 2019]{zanette2019tighter}
Zanette, A. and Brunskill, E. (2019).
\newblock Tighter problem-dependent regret bounds in reinforcement learning
  without domain knowledge using value function bounds.
\newblock {\em arXiv preprint arXiv:1901.00210}.

\end{thebibliography}
\end{document}